\documentclass{article}

\PassOptionsToPackage{authoryear,round}{natbib}
 \usepackage[preprint]{neurips_2025}

\usepackage[utf8]{inputenc} 
\usepackage[T1]{fontenc}    
\usepackage{hyperref}       
\usepackage{url}            
\usepackage{booktabs}       
\usepackage{amsfonts}       
\usepackage{nicefrac}       
\usepackage{microtype}      
\usepackage{xcolor}         
\usepackage{graphicx}
\usepackage{svg}
\usepackage{colortbl}
\usepackage{enumitem}
\usepackage{subcaption}
\usepackage{amsthm}         
\newtheorem{theorem}{Theorem}[section]

\theoremstyle{remark}

\hypersetup{
    colorlinks=true,   
    citecolor=neuripsaccentblue,    
    linkcolor=neuripsaccentblue,    
    urlcolor=neuripsaccentblue      
}

\title{LongSeeker: Elastic Context Orchestration for Long-Horizon Search Agents}

\author{%
  \textbf{Yijun Lu\textsuperscript{1,*}}, \textbf{Rui Ye\textsuperscript{1,*,\#,\textdagger}},  \textbf{Yuwen Du\textsuperscript{1}}, \textbf{Jiajun Wang\textsuperscript{1}}, \textbf{Songhua Liu\textsuperscript{1,\textdagger}}, \textbf{Siheng Chen\textsuperscript{1,\textdagger}} \\
  \textsuperscript{1}Shanghai Jiao Tong University, \textsuperscript{*}Equal Core Contributions, \textsuperscript{\#}Project Lead \\
  \textsuperscript{\textdagger}Corresponding Authors: \{yr991129, liusonghua, sihengc\}@sjtu.edu.cn
}

\begin{document}

\maketitle

\begin{abstract}
Long-horizon search agents must manage a rapidly growing working context as they reason, call tools, and observe information.
Naively accumulating all intermediate content can overwhelm the agent, increasing costs and the risk of errors.
We propose that effective context management should be adaptive: parts of the agent’s trajectory are maintained at different levels of detail depending on their current relevance to the task.
To operationalize this principle, we introduce \textbf{Context-ReAct}, a general agentic paradigm for \emph{elastic context orchestration} that integrates reasoning, context management, and tool use in a unified loop.
Context-ReAct provides five atomic operations: \emph{Skip}, \emph{Compress}, \emph{Rollback}, \emph{Snippet} and \emph{Delete}, which allow the agent to dynamically reshape its working context, preserving important evidence, summarizing resolved information, discarding unhelpful branches, and controlling context size.
We prove that the Compress operator is expressively complete, while the other specialized operators provide efficiency and fidelity guarantees that reduce generation cost and hallucination risk.
Building on this paradigm, we develop \textbf{LongSeeker}, a long-horizon search agent fine-tuned from Qwen3-30B-A3B on 10k synthesized trajectories.
Across four representative search benchmarks, LongSeeker achieves \textbf{61.5\%} on BrowseComp and \textbf{62.5\%} on BrowseComp-ZH, substantially outperforming Tongyi DeepResearch (43.2\% and 46.7\%) and AgentFold (36.2\% and 47.3\%).
These results highlight the potential of adaptive context management, showing that agents can achieve more reliable and efficient long-horizon reasoning by actively shaping their working memory.

  \medskip
  \begin{flushleft}
    \begin{tabular}{@{}ll@{}}
      \includegraphics[width=1em]{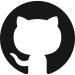} \quad \textbf{Code} & \href{https://github.com/PolarSeeker/LongSeeker}{https://github.com/PolarSeeker/LongSeeker} \\
      \includegraphics[width=1em]{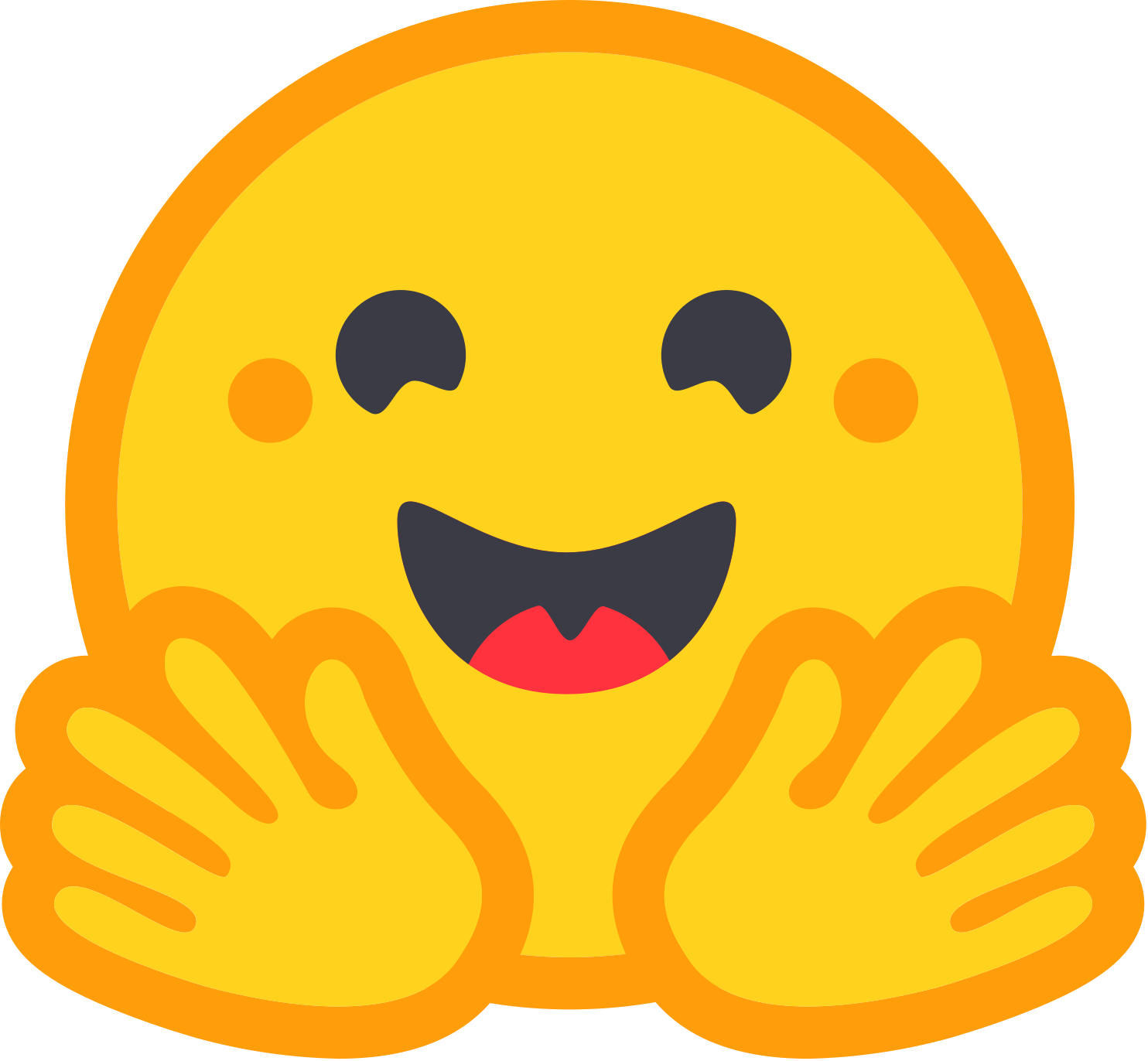} \quad \textbf{Model} & \href{https://huggingface.co/PolarSeeker/LongSeeker-30B-SFT}{https://huggingface.co/PolarSeeker/LongSeeker-30B-SFT} \\
    \end{tabular}
  \end{flushleft}
\end{abstract}

\begin{figure}[!h]
    \centering
    \vspace{-4mm}
    \includegraphics[width=1.0\linewidth]{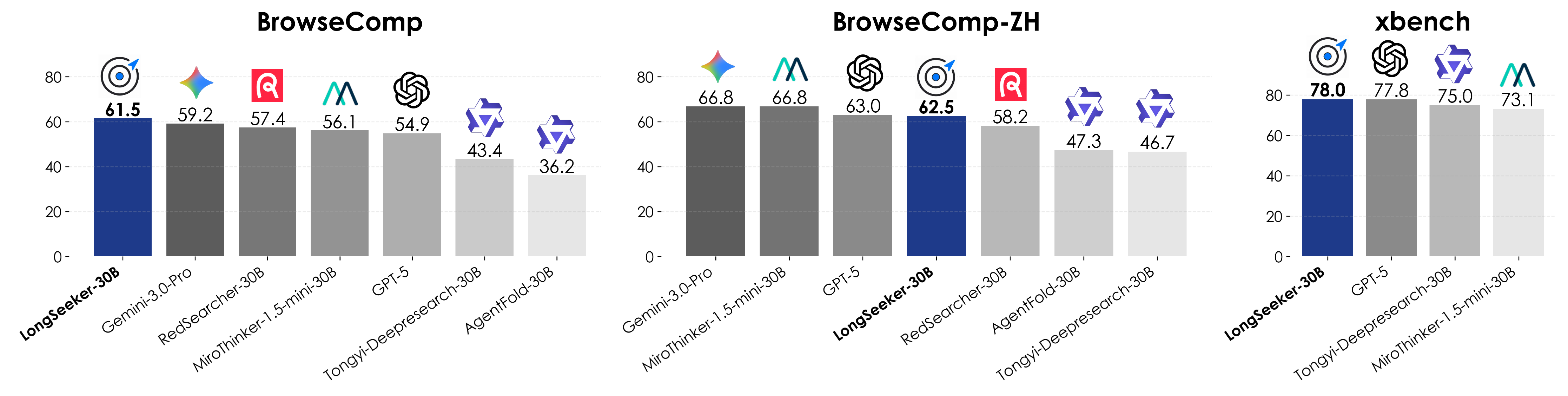}
    \vspace{-3mm}
    \caption{LongSeeker-30B delivers strong results on challenging long-horizon benchmarks, matching or surpassing several foundation models and search agents.}
    \label{fig:teaser}
\end{figure}

\section{Introduction}

The emergence of search agents has transformed how humans retrieve and synthesize information from the web.
Tasks once requiring manual query iteration can now be delegated end-to-end to an AI agent via a single instruction.
Agentic search has thus become a cornerstone capability pursued by AI labs, exemplified by the trajectory from OpenAI's Deep Research systems~\citep{dr} to today's top-tier large language models supporting multi-step, tool-augmented reasoning built upon the ReAct paradigm~\citep{yao2023reactsynergizingreasoningacting}.

However, long-horizon search agents built on the ReAct paradigm face an inherent context bottleneck: as observations, reasoning traces, and tool calls accumulate, the working context becomes increasingly noisy, redundant, and eventually too long to retain in full. 
Existing remedies remain partial: sliding-window truncation is importance-agnostic~\citep{miromindteam2026mirothinker17}; 
threshold-triggered re-starting disrupts reasoning continuity~\citep{deepseekai2025deepseekv32}; 
periodic summarization suffers from fixed-granularity compression and accumulating abstraction errors~\citep{zhou2025mem1learningsynergizememory,yu2025memagentreshapinglongcontextllm, lu2025scalingllmmultiturnrl}; 
and proactive curation is still limited in where or how it can intervene~\citep{ye2025agentfold,yao2026arcactivereflectiondrivencontext}. 
As a result, existing methods cannot provide precise, on-demand control over the evolving shape of the agent's context.

Addressing these, our key insight is that effective context management requires an \textit{elastic} working context, one that can compress, preserve, discard, and restructure different parts of the agent’s context according to its current state.
That is, during long-horizon search, information should exist in different forms as the task evolves: fresh evidence may need to remain intact for verification, resolved evidence can be distilled into conclusions, precision-critical details may survive as snippets, and failed branches should be removed or rolled back.
This \textit{state-dependent fidelity} ensures that the agent maintains the right level of detail for each part of its agentic reasoning trajectory.

Following this insight, we propose \textbf{Context-ReAct}, a general agentic paradigm for \emph{elastic context orchestration} in long-horizon search agents.
At each reasoning turn, the agent jointly produces its reasoning trace, a set of context meta-operations, and the next tool call in a single autoregressive pass. 
The meta-operations are applied before the next observation is appended, allowing the agent to actively determine \emph{when} to update its context, \emph{where} in the trajectory to intervene, and \emph{how} each part of the past should be represented.

Based on this paradigm, we design a meta-operation vocabulary consisting of five atomic actions.
(1) \emph{Skip} leaves the context unchanged when it is already compact and informative.
(2) \emph{Compress} replaces any contiguous range of historical steps with an abstractive summary. 
(3) \emph{Snippet} preserves an exact substring from a step, retaining precision-critical evidence such as numbers, entity names, quotations, or code without abstractive distortion.
(4) \emph{Delete} removes a step which no longer carries residual value.
(5) \emph{Rollback} abandons an unproductive branch by reverting the context to an earlier state while recording the reason for backtracking and any transferable insight.
Together, these operations maintain a multi-resolution working context in which different parts of the context can remain verbatim, be compressed, be partially quoted, be removed, or be structurally rolled back. 
Although simple, this operation set is proved to be \emph{expressively complete}: 
\textsc{Compress} alone can simulate every other operation in principle, while the specialized operators reduce generation cost and hallucination risk through explicit efficiency and fidelity guarantees.

To instantiate and evaluate Context-ReAct, we build \textbf{LongSeeker}, a long-horizon search agent fine-tuned from Qwen3-30B-A3B on 10k synthesized search trajectories.
We evaluate LongSeeker on four representative search benchmarks: BrowseComp~\citep{wei2025browsecompsimplechallengingbenchmark}, BrowseComp-ZH~\citep{zhou2025browsecompzhbenchmarkingwebbrowsing}, xbench~\citep{chen2025xbench}, and GAIA~\citep{mialon2023gaia}.
LongSeeker achieves scores of \textbf{61.5\%} and \textbf{62.5\%} on BrowseComp and BrowseComp-ZH respectively, significantly outperforming competitive baselines such as Tongyi DeepResearch (43.2\% and 46.7\%)~\citep{tongyideepresearchteam2025tongyideepresearchtechnicalreport} and AgentFold (36.2\% and 47.3\%)~\citep{ye2025agentfold}.
These results suggest that elastic context orchestration is a scalable path toward more capable long-horizon agents, shifting context management from a peripheral engineering heuristic to a core component of agentic reasoning.

Our main contributions are:
\begin{itemize}[leftmargin=*, topsep=2pt, itemsep=2pt]
  \item Paradigm. We propose Context-ReAct, a general agentic paradigm for elastic context orchestration that lets agents decide \emph{when}, \emph{where}, and \emph{how} to reshape their working context during ReAct-style search.
  \item Operations. We introduce five meta-operations, \emph{Skip}, \emph{Compress}, \emph{Rollback}, \emph{Snippet}, and \emph{Delete}, forming an expressively complete yet efficient operation set for multi-resolution context control.
  \item Experiments. We train \textbf{LongSeeker} on 10k synthesized trajectories and achieve \textbf{61.5\%} on BrowseComp and \textbf{62.5\%} on BrowseComp-ZH, outperforming strong long-horizon search baselines.
\end{itemize}

\section{Related Work}

\noindent \textbf{Search Agents.}
LLM-based search agents have transformed information retrieval from static query-response matching into dynamic, multi-step reasoning processes. Central to this transformation is the ReAct paradigm~\citep{yao2023reactsynergizingreasoningacting}, which structures agent behavior as an iterative cycle of reasoning, action execution, and observation integration. OpenAI's Deep Research~\citep{openai2025deepresearch} pioneers the fully closed-source path, followed by a series of proprietary agents; meanwhile, open-source efforts such as WebSailor~\citep{li2025websailornavigatingsuperhumanreasoning} and Tongyi DeepResearch~\citep{tongyideepresearchteam2025tongyideepresearchtechnicalreport} push capabilities forward through large-scale trajectory synthesis and post-training optimization. Yet these advances retain a fundamental limitation: they follow the conventional ReAct pattern of unconditionally accumulating observations, causing progressive degradation in context quality and heightened risk of exceeding context windows during extended tasks.

\noindent \textbf{Context Management for Agents.}
Managing growing context has attracted considerable recent attention, with existing methods falling into four categories. \emph{Sliding-window} heuristics such as keep-last-$k$---adopted by the MiroThinker series~\citep{miromindteam2026mirothinker17}---retain only recent steps and discard older content regardless of importance, while \emph{discard-all} variants flush the entire context once thresholds are reached, as in DeepSeek-V3.2~\citep{deepseekai2025deepseekv32} and GLM-4.7~\citep{zhipuai2025glm47}. \emph{Periodic summarization} methods such as MEM1~\citep{zhou2025mem1learningsynergizememory} train agents to maintain compact internal states and achieve strong multi-hop QA performance, while MemAgent~\citep{yu2025memagentreshapinglongcontextllm} processes documents in segments with fixed-size memory buffers. \emph{Proactive curation} approaches such as AgentFold~\citep{ye2025agentfold} and ARC~\citep{yao2026arcactivereflectiondrivencontext} enable agents to actively decide what and when to compress, yet they still lack surgical operations and cannot revisit earlier history to purge outdated content.

\noindent \textbf{Our Approach.}
Context-ReAct advances beyond prior work by defining a formally \emph{complete} set of five atomic meta-operations---\emph{Skip}, \emph{Compress}, \emph{Rollback}, \emph{Snippet} and \emph{Delete}---that are co-generated with standard tool calls at every step. Unlike fixed-rule truncation, each decision is content-aware; unlike periodic summarization, operations are invoked only when necessary; unlike proactive curation, our operation set is formally proven complete (Section~3) and spans the full spectrum from lossless extraction to structural backtracking. All decisions are learned end-to-end from synthesized long-horizon search trajectories.

\section{Method}
\begin{figure}[t]
  \centering
  \includegraphics[width=0.95\linewidth]{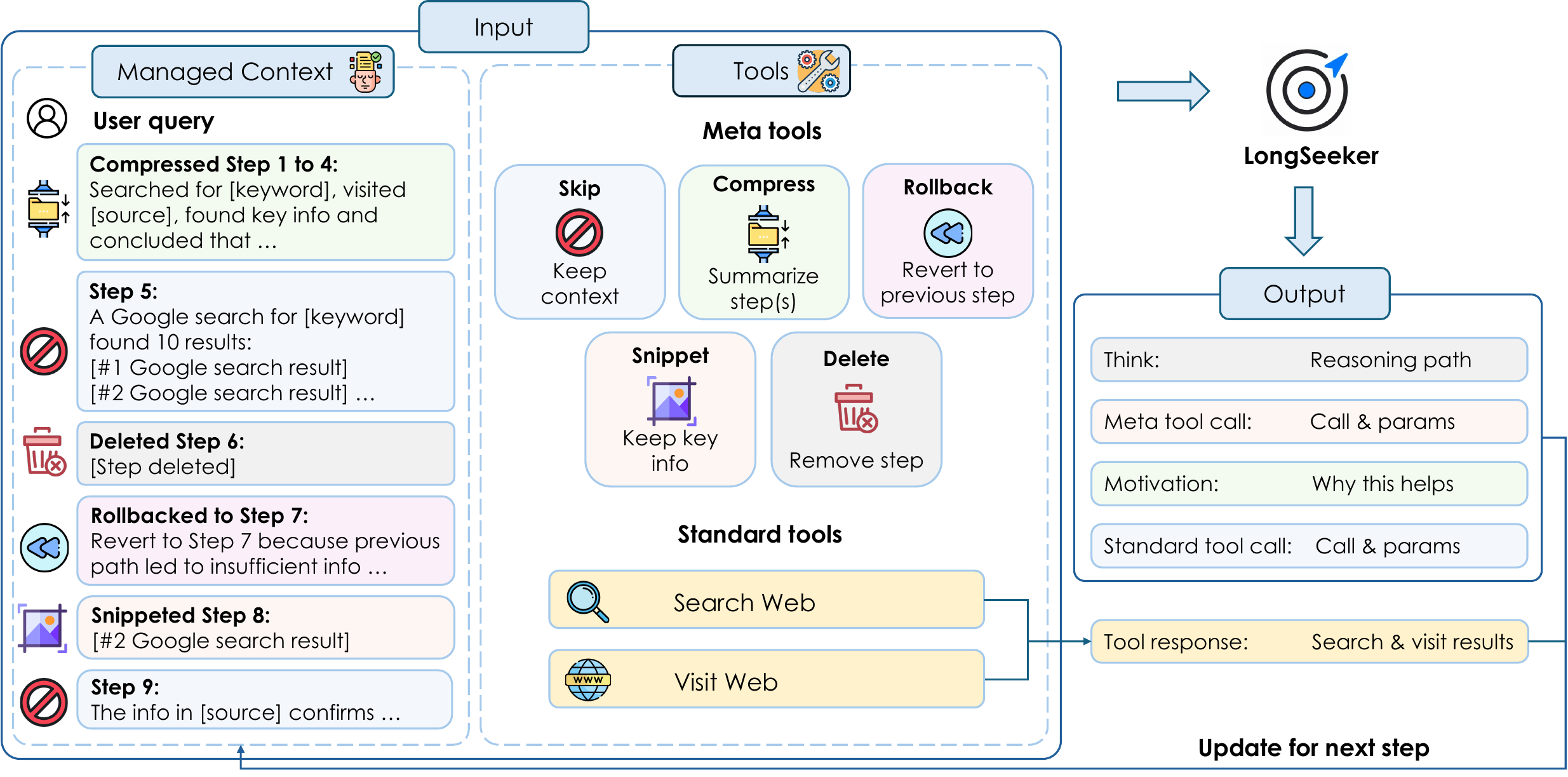}
  \caption{Overview of the \textbf{Context-ReAct} paradigm. Unlike standard ReAct, which passively accumulates history, and unlike prior proactive curation methods~\citep{ye2025agentfold, yao2026arcactivereflectiondrivencontext} that operate at a coarse granularity, Context-ReAct introduces a \emph{complete} and \emph{fine-grained} meta-action layer. At each step, the agent co-generates meta-operations (\textsc{Skip}, \textsc{Compress}, \textsc{Rollback}, \textsc{Snippet}, \textsc{Delete}) alongside standard tool calls, enabling elastic context orchestration that spans the full spectrum from lossless extraction to structural backtracking.}
  \label{fig:meta_react_overview}
\end{figure}

We present \textbf{Context-ReAct}, a general paradigm that augments the standard ReAct loop with an explicit \emph{meta-action} layer for on-demand context management. As illustrated in Figure~\ref{fig:meta_react_overview}, our approach enables the agent to actively curate its working memory by co-generating meta-operations alongside standard reasoning and tool calls. Section~\ref{sec:formulation} establishes the formal definition; Section~\ref{sec:ops} defines the five atomic meta-operations; Section~\ref{sec:completeness} proves expressive completeness and Section~\ref{sec:training} details the data synthesis and training pipeline.

\subsection{Agentic Paradigm}
\label{sec:formulation}

In this section, we introduction our proposed Context-ReAct with a formal definition.

\textbf{Standard ReAct.}
In the standard ReAct paradigm~\citep{yao2023reactsynergizingreasoningacting}, each step $S_i^{\mathrm{std}}$ is defined as
\begin{equation}
  S_i^{\mathrm{std}} = (r_i,\; c_i,\; o_i),
\end{equation}
where $r_i$ is the chain-of-thought reasoning trace, $c_i$ is the tool call, and $o_i$ is the environment observation returned by the tool. The context history at time $t$ is the concatenation $H_t = [S_1^{\mathrm{std}}, \ldots, S_t^{\mathrm{std}}]$, which grows monotonically and without bound under this \emph{append-only} design. As irrelevant observations accumulate, the signal-to-noise ratio of $H_t$ degrades, and $|H_t|$ eventually risks exceeding the model's context limit.

\textbf{Context-ReAct.}
To preserve the generality of the standard ReAct paradigm while equipping the agent with dynamic context management, we augment each step by inserting several \emph{meta operations} between the reasoning trace and the standard tool call.
The resulting step structure is
\begin{equation}
  S_i^{\mathrm{meta}} = (r_i,\; M_i,\; c_i,\; o_i),
\end{equation}
where $M_i = [\mathit{op}_1^{(i)}, \mathit{op}_2^{(i)}, \ldots, \mathit{op}_k^{(i)}]$ is a list of meta-operations generated by the agent to transform the context \emph{before} the next step begins.
Formally, the effective context $H_t'$ used at step $t{+}1$ is
\begin{equation}
  H_t' \;=\; T(H_t,\; M_t),
\end{equation}
where $T$ is the composition of the individual operations in $M_t$, each drawn from the primitive set (defined in Section~\ref{sec:ops}). This mechanism enables the agent to maintain a compact and relevant working memory ($|H_t'| \ll |H_t|$ in typical long-horizon tasks) without any external trigger or architectural modification.

The design principle is that meta-operations are \emph{co-generated} with the reasoning trace and tool call in a single, end-to-end generation step, rather than being triggered by an external heuristic such as a length threshold.
The agent therefore learns \emph{when}, \emph{where}, and \emph{how} to intervene in its own context as an integral part of its policy; see Figure~\ref{fig:meta_react_overview} (right) for illustration.

\subsection{Atomic Meta-Operations}
\label{sec:ops}

To enable flexible and task-aware management of the working context, Context-ReAct equips the agent with a set of atomic meta-operations.
These operations define the full set of primitive actions that can be applied to the context at each reasoning step. Throughout, let $H = [S_1, \ldots, S_n]$ denote the current history.

\textbf{(1) \textsc{Skip}}
is the identity operator. That is, the agent takes no action on the context:
\begin{equation}
  \textsc{Skip}(H) \;=\; H,
\end{equation}
which is issued when the current context is already compact, incurring zero additional generation overhead.

\textbf{(2) \textsc{Compress}}
performs abstractive summarization over any contiguous range of steps $[a,b]$ ($a \leq b$), replacing the context with a summarized step $S_{a:b}$, where $\Sigma$ is the summarized string:
\begin{equation}
  \textsc{Compress}(H,\,a,\,b,\,\Sigma) \;=\;
  [S_1,\;\ldots,\;S_{a-1},\;S_{a:b}=\Sigma,\;
   S_{b+1},\;\ldots,\,S_n].
\end{equation}
Crucially, $[a,b]$ need not be a recent window~\citep{ye2025agentfold}: the agent can \emph{retroactively} recognize that steps from early in the trajectory have become compressible. 
For example, earlier searches may no longer need to be preserved in full once their useful information has been captured by later evidence, allowing the agent to compress them on the fly.
This look-back flexibility is unavailable in sliding-window approaches, which can only truncate from one end of the history.

\textbf{(3) \textsc{Rollback}}
reverts the context to step $k$ by discarding all subsequent steps $S_{k}, \ldots, S_n$ and appending a summarized step that records the reason for backtracking and any transferable insight:
\begin{equation}
  \textsc{Rollback}(H,\,k,\,\Sigma) \;=\;
  [S_1,\;\ldots,\;S_k=\Sigma].
\end{equation}
\textsc{Rollback} models the structural intuition of branch abandonment in tree-based search (DFS/MCTS)~\citep{shi2025montecarloplanninglarge}: when the agent recognizes that a reasoning path has reached a dead end, it discards the failed sub-trajectory while preserving its causal explanation, preventing the same mistake from being repeated.

\textbf{(4) \textsc{Snippet}}
replaces the observation $o_k$ of the $k$-th step with the verbatim substring delimited by the anchor strings \textit{pre} and \textit{suf}:
\begin{equation}
  \textsc{Snippet}(H,\,k,\,\textit{pre},\,\textit{suf}) \;=\;
  \bigl[S_1,\;\ldots,\;(r_k,\,c_k,\,o_k[\textit{pre}{:}\textit{suf}]),
  \;\ldots,\,S_n\bigr].
\end{equation}
Unlike generative summarization, \textsc{Snippet} is \emph{lossless} with respect to the retained segment: it performs pointer-based substring extraction rather than token regeneration, saving token cost and preventing hallucination of precise numerical values, entity names, URLs, or code that must be carried forward exactly.

\textbf{(5) \textsc{Delete}}
removes the $k$-th step, discarding its reasoning trace, tool call, and observation:
\begin{equation}
  \textsc{Delete}(H,\,k) \;=\;
  \bigl[S_1,\;\ldots,\;S_{k-1},\;S_{k+1},\;\ldots,\,S_n\bigr].
\end{equation}
This operation is appropriate when an entire step is uninformative and leaves no useful trace---e.g., a failed or redundant query whose result, reasoning, and call all warrant complete removal to reduce noise.

\textbf{Composite application.}
Multiple meta-operations may be composed within a single step by listing them sequentially in $M_i$. This compositionality allows \textbf{LongSeeker} to, for instance, \textsc{Delete} a noisy interaction step and simultaneously \textsc{Compress} a longer historical segment in one step.

Figure~\ref{fig:step_example} illustrates the combined effect of these meta-operations on a live trajectory (see Appendix~\ref{app:case_study} for a complete example). 
Rather than receiving the full raw trajectory, the model sees a compact, curated view of its history. 
As shown in the left half of the figure, Steps~1--4 have been consolidated via \textsc{Compress} into a single summary sentence that preserves the essential findings while discarding verbose intermediate observations. Step~5 remains unchanged via \textsc{Skip}. Step~6 has been \textsc{Delete}d to eliminate noise from a redundant query. Most significantly, the trajectory has been \textsc{Rollback}ed to Step~7, discarding an unproductive sub-chain of exploration. The resulting managed context is minimal yet information-dense, ensuring the model can focus on the remaining open sub-questions.

\begin{figure}[t]
  \centering
  \includegraphics[width=\linewidth]{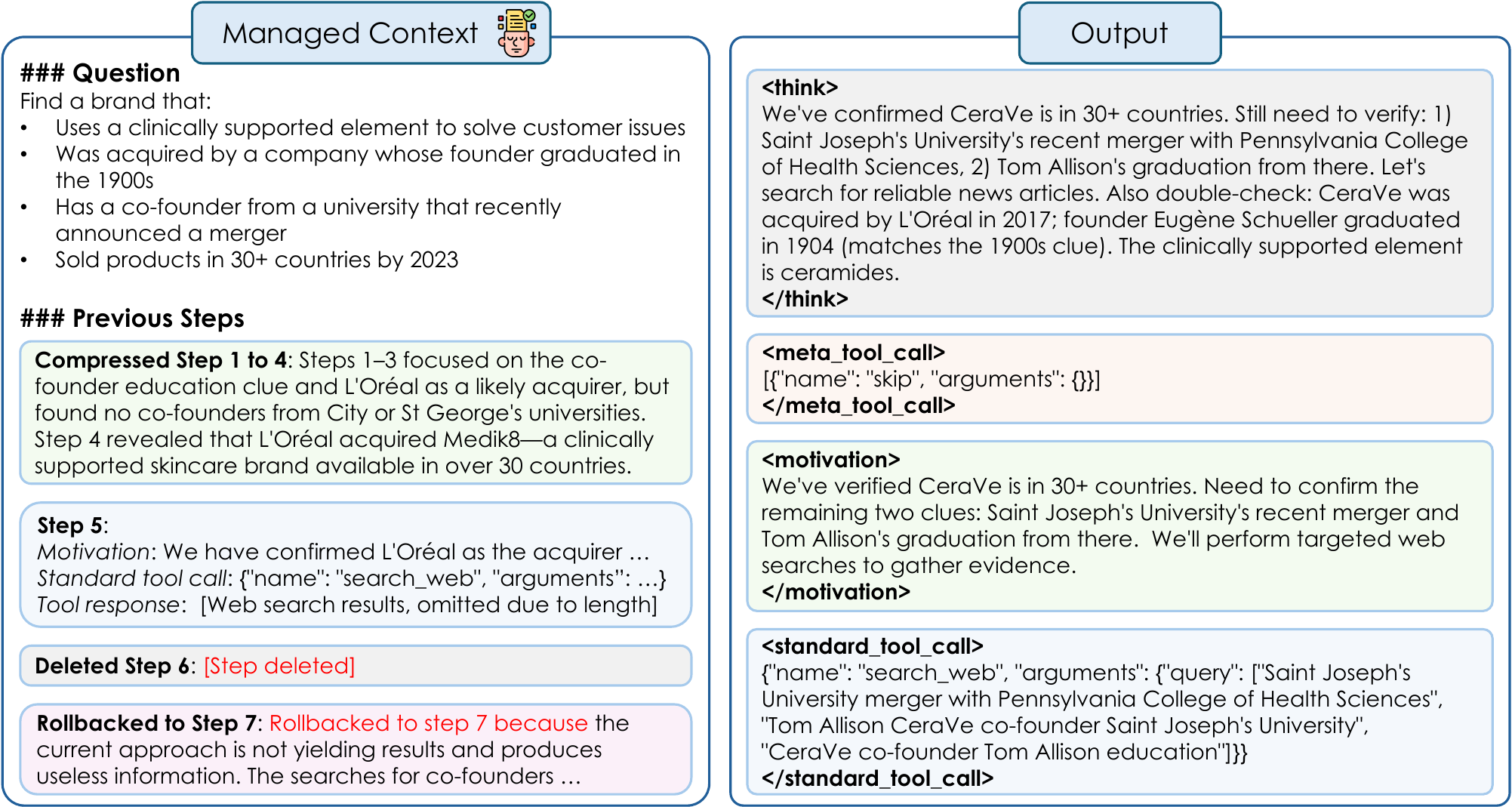}
  \caption{%
    \textbf{Managed context and structured output at a single Context-ReAct step.}
    \emph{Left:} The curated context after applying meta-operations to the raw trajectory. Steps~1--4 are consolidated via \textsc{Compress} into a summary preserving essential findings. Step~5 remains unchanged via \textsc{Skip}. Step~6 is \textsc{Delete}d to eliminate noise from a redundant query. The trajectory is \textsc{Rollback}ed to Step~7, discarding unproductive exploration. The resulting context is minimal yet information-dense.
    \emph{Right:} The four-field structured output containing reasoning, meta-operations, motivation, and the standard tool call.
  }
  \label{fig:step_example}
\end{figure}

\subsection{Expressive Completeness and Principled Redundancy}
\label{sec:completeness}

We now analyze their theoretical properties with a simple d ceduction to justify both the completeness and practical utility of the Context-ReAct action set.

\begin{theorem}[Expressive Completeness]
\label{thm:completeness}
The meta-action set $\mathcal{O} = \{\textsc{Skip}, \textsc{Compress}, \textsc{Rollback}, \textsc{Snippet}, \textsc{Delete}\}$ is expressively complete: for any $H_{\mathrm{in}},\,H_{\mathrm{target}} \in \mathcal{H}$, there exists a finite sequence of operations from $\mathcal{O}$ that transforms $H_{\mathrm{in}}$ into $H_{\mathrm{target}}$.
\end{theorem}

\begin{proof}
It suffices to show that \textsc{Compress} alone is a universal string rewriting operator over $\mathcal{H}$. By definition, $\textsc{Compress}(H, 1, |H|, \Sigma)$ replaces the \emph{entire} history $H$ with an arbitrary string $\Sigma \in \mathcal{V}^*$. Setting $\Sigma = H_{\mathrm{target}}$ yields $H' = H_{\mathrm{target}}$ in a single operation. Since a single element of $\mathcal{O}$ can reach any target from any source, the full set $\mathcal{O}$ is trivially complete.
\end{proof}

Although \textsc{Compress} alone is sufficient for theoretical completeness, the remaining four operators provide practical structure by guiding the agent to manage context more efficiently and reliably. 
Each operator addresses a distinct operational need:

\begin{itemize}[leftmargin=1.5em, itemsep=4pt]

\item \textbf{\textsc{Skip}: identity.}
$\textsc{Skip}(H) \equiv \textsc{Compress}(H,1,|H|,H)$ indicates preservation of the entire history.

\item \textbf{\textsc{Rollback}: structural search prior.}
$\textsc{Rollback}(H,k,\Sigma) \equiv \textsc{Compress}(H,k,|H|,\Sigma)$, but framing it as ``rollback to step $k$'' gives the model a clearer inductive bias: it should discard incorrect reasoning branches, similar to backtracking in tree search. This helps the agent learn \emph{when} to abandon a failed path, rather than treating it as a generic compression.

\item \textbf{\textsc{Snippet}: extraction.}
$\textsc{Snippet}(H,k,\textit{pre},\textit{suf}) \equiv \textsc{Compress}(H,k,\,k,\,o_k[\textit{pre}{:}\textit{suf}])$, but generative compression is stochastic and lossy. \textsc{Snippet} guarantees \emph{exact} retention of critical content via pointer-based extraction, which is critical when the retained segment contains numerical values, entity names, or code that must be reproduced verbatim in later reasoning steps.

\item \textbf{\textsc{Delete}: complete step removal.}
$\textsc{Delete}(H,k) \equiv \textsc{Compress}(H,k,k,\emptyset)$, but framing it as ``delete step $k$'' makes the operation explicit: the entire step $k$ is removed from the context.

\end{itemize}

In summary, the five atomic meta-operations can be interpreted as specialized transformations over distinct subspaces of the space of all possible context.
By partitioning the context transformation space in this way, the agent can maintain a compact, relevant, and reliable working context throughout long-horizon reasoning.
This structured decomposition also aligns with the Minimum Description Length principle~\citep{Rissanen1978,grunwald2004tutorial}, providing a rationale for why these specialized operators are beneficial in practice.

\subsection{Data Synthesis and Training}
\label{sec:training}

\textbf{Trajectory synthesis.}
Training \textbf{LongSeeker} via \textbf{Context-ReAct} requires trajectories that contain not only correct final answers but also high-quality \emph{context management decisions} at intermediate steps---a supervision signal absent from all existing datasets. We construct a corpus of $10{,}000$ annotated trajectories through a two-stage pipeline.

\textit{Stage 1: Seed question collection.}
We sample $10{,}000$ complex, multi-hop questions from OpenSeeker~\citep{du2026openseeker,du2026openseekerv2}, comprising $9{,}000$ English and $1{,}000$ Chinese questions, filtering for questions that require substantive multi-step reasoning to answer.

\textit{Stage 2: Context-ReAct trajectory rollout.}
Each question is solved by DeepSeek\,V3.2~\citep{deepseekai2025deepseekv32} acting as the teacher model and operating under the full Context-ReAct paradigm. 
At every step, the teacher directly generates the complete four-field structured output---\texttt{<think>}, \texttt{<meta\_tool\_call>}, \texttt{<motivation>}, and \texttt{<standard\_tool\_call>}---in a single pass, producing context management decisions and the next tool call jointly. Trajectories with correct format constitute the final training set.

\textbf{Supervised fine-tuning.}
We fine-tune Qwen\,3\,30B-A3B~\citep{yang2025qwen3technicalreport} on the annotated corpus via standard next-token prediction:
\begin{equation}
  \mathcal{L}_{\mathrm{SFT}} \;=\;
  -\sum_{t=1}^{T}\sum_{j}
  \log p_\theta\!\left(
    x_j^{(t)} \;\middle|\; x_{<j}^{(t)},\; H_{t-1}'
  \right),
\end{equation}
where $x^{(t)}$ is the full structured output at step $t$---including the chain-of-thought, meta-tool call, motivation, and standard tool call---and $H_{t-1}'$ is the context after applying the meta-operations from the previous step. Computing the loss over the \emph{entire} structured output forces the model to jointly learn \emph{which} meta-operation to invoke, \emph{when} to invoke it, and \emph{how} to use the standard tool given the current context.

\section{Experiments}
\subsection{Experimental Setup}

\textbf{Evaluations.}
We evaluate our \textbf{LongSeeker} on four key benchmarks spanning targeted information-seeking and general agent capabilities. \textbf{BrowseComp}~\citep{wei2025browsecomp} and \textbf{BrowseComp-ZH}~\citep{zhou2025browsecompzh} evaluate multi-step navigation and hard information retrieval in English and Chinese, respectively (sampling 200 questions from each benchmark due to resource constraints). \textbf{xbench}~\citep{chen2025xbench} assesses complex deep research capabilities including planning, reasoning, and synthesis across profession-aligned real-world tasks.
Finally, \textbf{GAIA}~\citep{mialon2023gaia} (text-only subset) evaluates general agent capabilities requiring combined web browsing, tool use, and multi-step reasoning.
We set the max tool call as 300 for all benchmarks. For BrowseComp and BrowseComp-ZH, we also apply the discard-all technique and allow for 5 rounds at maximum following MiroThinker~\citep{miromindteam2026mirothinker17}.

\textbf{Baselines.}
To validate the effectiveness of \textbf{LongSeeker}, we compare it against several state-of-the-art systems categorized into two groups: (1) \emph{foundation models with tools}, comprising frontier proprietary systems such as GPT-5~\citep{openai2025gpt5}, Gemini-3.0-Pro~\citep{google2025gemini3pro}, Claude-Opus-4.5~\citep{anthropic2025claudeopus45}, and Seed-2.0-Pro~\citep{bytedance2026seed20}, alongside open-weight models DeepSeek-V3.2~\citep{deepseekai2025deepseekv32} and GLM-4.7~\citep{zhipuai2025glm47}; and (2) \emph{search agents}, which serve as direct, comparable-scale benchmarks at 30B parameters, including MiroThinker series~\citep{miromindteam2026mirothinker17}, REDSearcher~\citep{zheng2026redsearcher}, IterResearch~\citep{chen2026iterresearch}, AgentFold~\citep{ye2025agentfold},  Tongyi-DeepResearch~\citep{tongyideepresearchteam2025tongyideepresearchtechnicalreport}, and OpenSeeker~\citep{du2026openseeker}. This diverse baseline set covers representative paradigms in contemporary agentic search. All baseline results are sourced from official publications or publicly available evaluation platforms.

\begin{table}[t]
\centering
\caption{Main results. \textbf{LongSeeker}, trained under the \textbf{Context-ReAct} paradigm, outperforms GPT-5 and Gemini-3.0-Pro on BrowseComp despite having only 30B parameters, highlighting the effectiveness and potential of Context-ReAct. Scores marked with $^{*}$ denote ReAct-based agents without context management on BrowseComp and BrowseComp-ZH, while ``--'' indicates unavailable or unknown results.}
\label{tab:main_results}
\resizebox{\textwidth}{!}{%

\begin{tabular}{@{}lcccccc@{}}
\toprule
\textbf{Model} & \textbf{Param} & \textbf{Training} & \textbf{BrowseComp} & \textbf{BrowseComp-ZH} & \textbf{xbench-2505} & \textbf{GAIA-text} \\
\midrule
\multicolumn{7}{c}{\emph{\textbf{Foundation Model with Tools}}} \\
\midrule
GPT-5 & - & - & 54.9 & 63.0 & 77.8 & 76.4 \\
Gemini-3.0-Pro & - & - & 59.2 & 66.8 & - & 74.8 \\
Claude-Opus-4.5 & - & - & 67.8 & 62.4$^{*}$ & - & 71.5 \\
Seed-2.0-Pro & - & - & 77.3 & 82.4 & - & 78.6 \\
DeepSeek-V3.2 & 671B & - & 67.6 & 65.0$^{*}$ & 78.0 & 75.1 \\
GLM-4.7 & 358B & - & 67.5 & 66.6$^{*}$ & 72.0 & - \\
\midrule
\multicolumn{7}{c}{\emph{\textbf{Search Agent}}} \\
\midrule
MiroThinker-1.7-mini & 30B & CPT + SFT + RL & 67.9 & 72.3 & - & 80.3 \\
MiroThinker-1.5-mini & 30B & CPT + SFT + RL & 56.1 & 66.8 & 73.1 & 72.0 \\
RedSearcher & 30B & CPT + SFT + RL & 57.4 & 58.2 & - & 80.1 \\
IterResearch & 30B & SFT + RL & 37.3 & 45.2 & 71.0 & 72.8 \\
AgentFold & 30B & SFT & 36.2 & 47.3 & - & 67.0 \\
Tongyi-DeepResearch & 30B & CPT + SFT + RL & 43.4$^{*}$ & 46.7$^{*}$ & 75.0 & 70.9 \\
OpenSeeker & 30B & SFT & 29.5$^{*}$ & 48.4$^{*}$ & 74.0 & - \\
\midrule
\textbf{LongSeeker} & 30B & SFT & \textbf{61.5} & \textbf{62.5} & \textbf{78.0} & \textbf{77.7} \\
\bottomrule
\end{tabular}%
}
\end{table}

\subsection{Results and Analysis}

\textbf{Main results.}
Table~\ref{tab:main_results} presents the primary evaluation results on BrowseComp and BrowseComp-ZH. \textbf{LongSeeker} achieves \textbf{61.5} on BrowseComp and \textbf{62.5} on BrowseComp-ZH, establishing strong performance among 30B-scale open-source search agents. Notably, LongSeeker exceeds MiroThinker-1.5-mini (56.1), Tongyi-DeepResearch (43.4), IterResearch (37.3), AgentFold (36.2), and OpenSeeker-v1 (29.5). Extending evaluation to xbench and GAIA, LongSeeker achieves 78.0 on xbench-2505 and 77.7 on GAIA-text. These scores confirm that the benefits of Elastic Context Orchestration generalize beyond purely information-seeking tasks to broader agent capabilities, with LongSeeker establishing competitive performance across diverse benchmark suites.

\textbf{Context growth dynamics.}
To empirically validate the efficacy of our context management paradigm, we trace the trajectory length and corresponding context token count across 200 questions sampled from BrowseComp. As depicted in Figure~\ref{fig:growth_curve}, we plot the surviving trajectories at each turn alongside their average accumulated context tokens. Unlike ReAct-based DeepSeek-V3.2, which suffers from rapid, unbounded context inflation as observations are passively appended step-by-step, \textbf{LongSeeker} maintains a remarkably stable and concise working memory. The token count initially scales with the problem depth but soon reaches a plateau, staying under 15k tokens even at extended horizons of 300 steps. 

This stabilized growth is a direct consequence of our \emph{complete} and \emph{fine-grained} meta-operations: rather than accumulating noise, the model learns to dynamically purge failed branches (\textsc{Rollback}), discard irrelevant retrievals (\textsc{Delete}), extract only essential snippets (\textsc{Snippet}), and abstract verbose history (\textsc{Compress}). Consequently, the context remains highly compact and information-dense. \textbf{LongSeeker} delivers competitive performance on long-horizon benchmarks while utilizing only a fraction of the underlying model's maximum 256k context window. This vast remaining capacity leaves ample headroom for tackling longer and more complex exploratory tasks. This confirms that the model has internalized how to efficiently deploy our atomic meta tools, retaining critical reasoning signals while minimizing distracting noise.

\textbf{Meta-operation usage.}
Figure~\ref{fig:distribution} shows the usage distribution of the five meta-operations of LongSeeker. We observe that LongSeeker effectively leverages the full set of atomic operations—\emph{Skip}, \emph{Compress}, \emph{Rollback}, \emph{Snippet}, and \emph{Delete}—across trajectories, suggesting that it has acquired a robust strategy for invoking and composing meta-operation to handle complex long-horizon tasks. This ability to coherently and purposefully navigate the context-manipulation space contributes to its strong overall performance.

We also observe a mild imbalance, where \emph{Snippet} and \emph{Delete} are used less frequently. This likely stems from the nature of long-horizon search: early in the process, it is difficult to confidently identify irrelevant information, so LongSeeker tends to preserve more context and adopts a conservative pruning strategy.

\begin{figure}[t]
    \centering

    \begin{subfigure}{0.48\linewidth}
        \centering
        \includegraphics[width=1.2\linewidth]{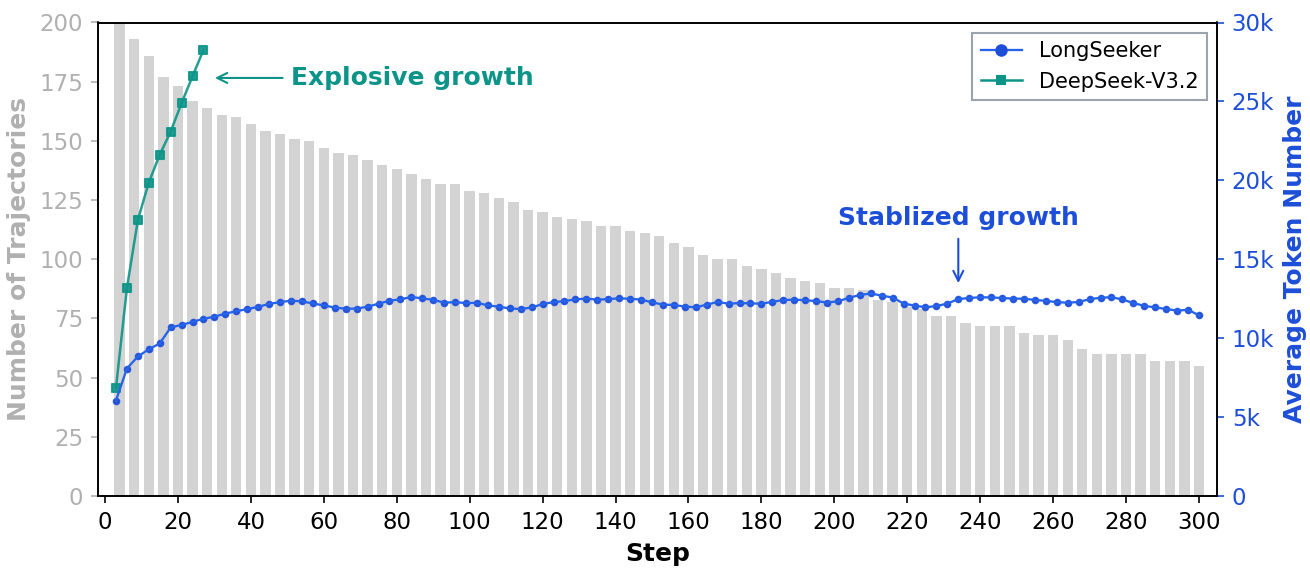}
        \caption{Context Growth Dynamics of LongSeeker}
        \label{fig:growth_curve}
    \end{subfigure}
    \hfill
    \begin{subfigure}{0.48\linewidth}
        \centering
        \includegraphics[width=0.7\linewidth]{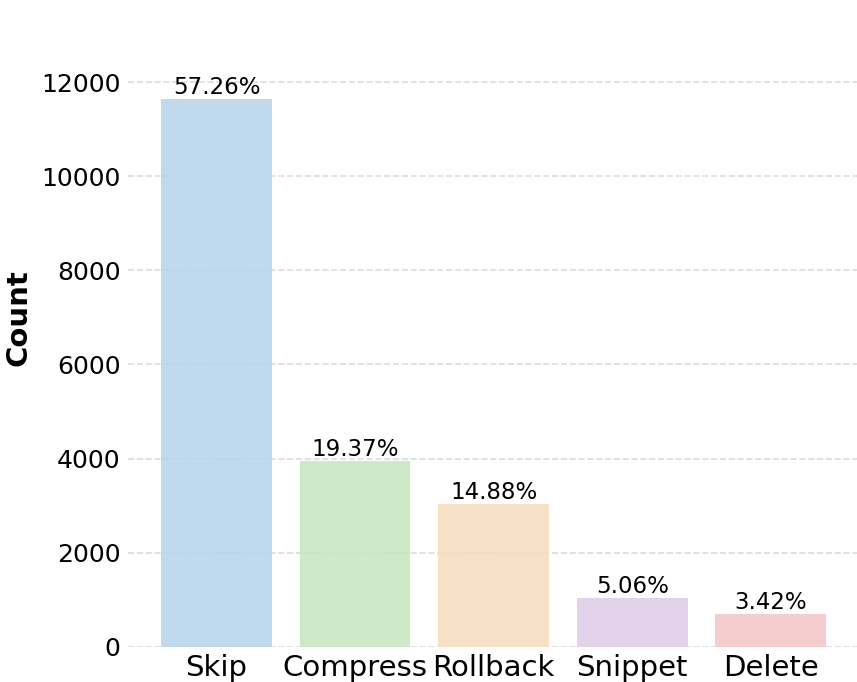}
        \caption{Meta-operation distribution of LongSeeker}
        \label{fig:distribution}
    \end{subfigure}
    \caption{Analysis of LongSeeker’s context management on 200 trajectories sampled from BrowseComp. (a) The average context token count remains stable and well bounded (plateauing around 15k tokens) over long horizons, in contrast to the explosive linear growth of ReAct-based DeepSeek-V3.2. The managed context is highly compact, utilizing a mere fraction of the LongSeeker's 256k capacity. (b) LongSeeker learns through training to utilize all five meta-operations and effectively invoke and compose them to solve long-horizon tasks.}

\end{figure}

\begin{figure}[t]
    \centering
    \includegraphics[width=0.8\linewidth]{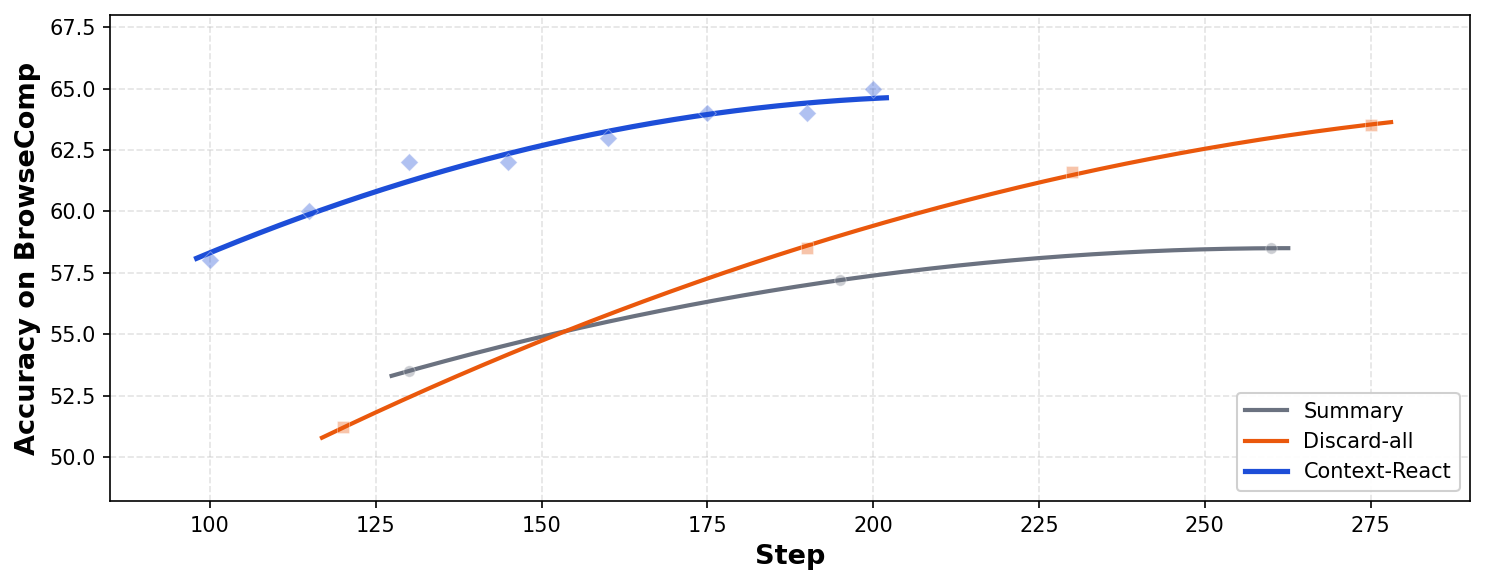}
    \caption{Effectiveness of the Context-ReAct paradigm on BrowseComp compared to other context management strategies. Context-ReAct achieves better performance under the same step budget.}
    \label{fig:ablation}
\end{figure}
\textbf{Comparison of Context Management Strategies.}
To evaluate the effectiveness of the Context-ReAct paradigm, we conduct controlled experiments on BrowseComp under a unified setup, where all methods are built upon the same base model (DeepSeek-V3.2) and share identical configurations. We compare Context-ReAct with two commonly used context management strategies used in DeepSeek-V3.2~\citep{deepseekai2025deepseekv32}: (1) \emph{Summary}, which compresses the overflowed trajectory into a summary and resumes the rollout from the condensed context; and (2) \emph{Discard-all}, which resets the context by removing all previous tool-call history, similar to reinitializing with a fresh context. As shown in Figure~\ref{fig:ablation}, Context-ReAct consistently achieves the highest performance under the same step budget, demonstrating its superior effectiveness in long-horizon tasks through elastic context control that enables adaptive management and utilization of context throughout extended trajectories.

\section{Conclusion}

We propose \textbf{Context-ReAct}, a general agentic paradigm for elastic context orchestration that enables agents to jointly generate reasoning, context meta-operations, and tool calls at each step. Context-ReAct defines five atomic operations---\emph{Skip}, \emph{Compress}, \emph{Rollback}, \emph{Snippet} and \emph{Delete}. Together, these operations provide a \emph{complete} and \emph{fine-grained} mechanism for multi-resolution control over the evolving working context. We train \textbf{LongSeeker-30B} based on this paradigm and demonstrate competitive performance on long-horizon search benchmarks, notably surpassing Tongyi DeepResearch and AgentFold on BrowseComp. Our experiments empirically validate that the application of these atomic tools yields significantly more compact and efficient context management compared to append-only or coarse-grained truncation approaches, enabling sustained high performance at long horizons with strong potential for harder tasks.

\textbf{Future work.}
Our current implementation leverages SFT on synthesized trajectories without rejection sampling or advanced exploration strategies.
One direction involves applying RL to optimize meta-operation usage, enabling agents to explore the action spaces.
Furthermore, we envision Context-ReAct as a foundational architectural paradigm rather than a search-specific solution.
Its core philosophy of state-dependent fidelity is inherently domain-agnostic, offering a scalable blueprint for other long-horizon challenges such as autonomous software engineering, large-scale legal discovery, and multi-modal scientific reasoning, where the ability to fluidly restructure massive working contexts is critical.

\medskip
{
\bibliographystyle{plainnat}
\bibliography{ref}
}

\clearpage
\appendix

\section{Case Study}
\label{app:case_study}

Figure~\ref{fig:case_study} shows the managed context after applying meta-operations, and Figure~\ref{fig:case_study_output} shows the corresponding structured output from \textbf{LongSeeker}.

\begin{figure*}[h]
  \centering
  \includegraphics[height=0.75\textheight]{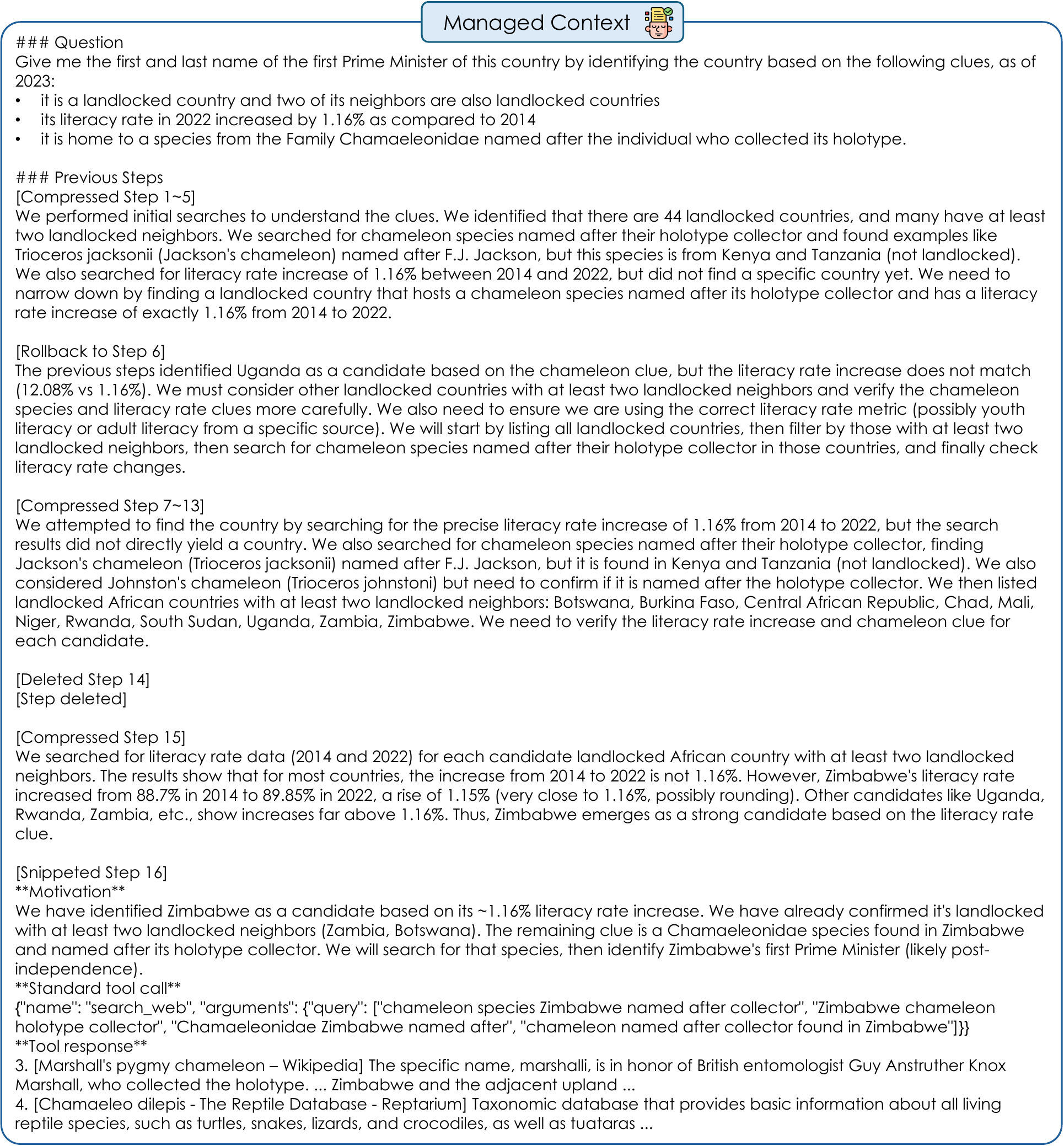}
  \caption{Complete case study showing managed context at a reasoning step. The trajectory demonstrates the combined effect of \textsc{Compress}, \textsc{Rollback}, \textsc{Delete}, and \textsc{Snippet} operations.}
  \label{fig:case_study}
\end{figure*}

\begin{figure*}[h]
  \centering
  \includegraphics[height=0.7\textheight]{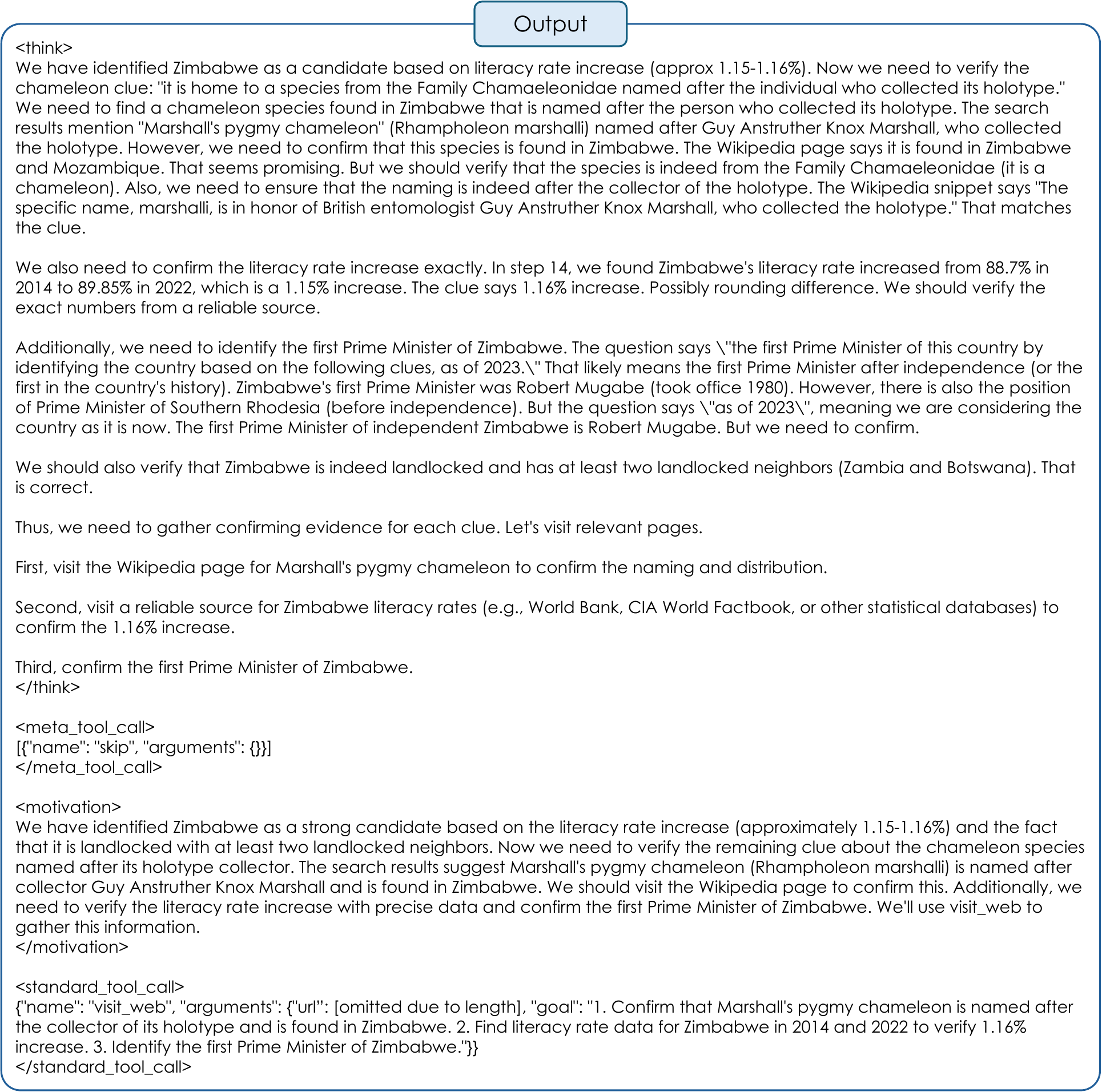}
  \caption{Complete structured output from \textbf{LongSeeker}, including reasoning, meta-tool calls, motivation, and standard tool call.}
  \label{fig:case_study_output}
\end{figure*}

\end{document}